\documentclass[11pt]{article}

\title{\Large\bf Online learning with exponential weights in metric spaces}
\author{Quentin Paris\footnote{HSE University, Faculty of Computer Science, Moscow, Russia. This work has been funded by the  Russian Academic Excellence Project '5-100'. Email:\url{qparis@hse.ru}}}
\date{}

\usepackage{amsthm,amssymb,amsmath}
\usepackage{color}
\usepackage{enumitem}
\usepackage{url}
\usepackage{bbold}
\usepackage{cases}
\usepackage{float}
\usepackage{eucal}
\usepackage{titlesec}
\titleformat{\section}
{\large\bf}{\thesection}{1em}{}
\titleformat{\subsection}
{\normalfont\bf}{\thesubsection}{1em}{}

\usepackage[square,numbers]{natbib}
\usepackage[pdftex]{hyperref}
\usepackage[hyperpageref]{backref}
\setcitestyle{authoryear,open={(},close={)}}

\usepackage{graphicx}
\usepackage[font=small,labelfont=bf]{caption}
\usepackage{wrapfig}

\numberwithin{equation}{section}
\usepackage{hyperref}
\hypersetup{
  colorlinks = true,
  urlcolor = blue, 
  linkcolor = blue,
  citecolor = blue}

\newtheorem{thm}{Theorem}[section]
\newtheorem{pro}[thm]{Proposition}
\newtheorem{lem}[thm]{Lemma}
\newtheorem{cor}[thm]{Corollary}
\newtheorem{defi}[thm]{Definition}
\newtheorem{rem}[thm]{Remark}
\newtheorem{exm}[thm]{Example}

\newcommand{\esp}{\mathbb{E}}

\newcommand{\R}{\mathbb R}

\newcommand{\e}{\varepsilon}

\begin{document}

\maketitle
\begin{abstract}
 This paper addresses the problem of online learning in metric spaces using exponential weights. We extend the analysis of the exponentially weighted average forecaster, traditionally studied in a Euclidean settings, to a more abstract framework. Our results rely on the notion of barycenters, a suitable version of Jensen's inequality and a synthetic notion of lower curvature bound in metric spaces known as the measure contraction property. We also adapt the online-to-batch conversion principle to apply our results to a statistical learning framework. 
\end{abstract}

%------------------
%------------------
\section{Introduction}
The problem of online convex optimization~\citep{CesLug06,Sha12,Haz16} has become a strandard model of online learning. Its simple and flexible formulation as a repeated game, devoid of distributional assumptions on the data, has proven effective in framing theoretically a number of online prediction tasks including online recommendation systems, online portfolio selection or network routing problems. Traditionally studied in the context of Euclidean spaces, less seems to be known when the decision space is a more general metric space, with potentially no linear structure. In this paper, we extend the analysis of the exponentially weighted average (\textsc{ewa}) forecaster to some geodesic metric spaces.

Motivations for this level of generality arise, for example, when the decision space is a smooth manifold. Such a scenario is routinely encountered in directional or shape statistics~\citep{Mar99} where observations take values in spheres, projective spaces or shape spaces. Computer vision and medical imaging deal with spaces of transformations that carry a Lie group structure~\citep{You19}. The space of real symmetric positive-definite matrices with Bures metric~\citep{Bur69}, Log-Euclidean metric~\citep{AFPA2007} or Log-Cholesky metric~\citep{Lin2019} have found important applications in quantum information theory or diffusion tensor imaging. Recent works have also demonstrated the use of hyperbolic spaces for applications in natural language processing~\citep{NiKi17}. Other situations of interest involve spaces which cannot be endowed with a formal Riemannian structure such as the $2$-Wasserstein space~\citep{Vil09,San15}, the Billera-Holmes-Vogtmann space of phylogenetic trees~\citep{bhv} or the Gromov-Wasserstein space~\citep{Mem11,Stu12}. For a detailed survey on data analysis in non-standard spaces, we also refer the reader to~\cite{HucElt20}. 

All the spaces mentioned above share the characteristic of being geodesic spaces~\citep{BriHaf99,BurBurIva01,AleKapPet19_invitation,AleKapPet19}, a structure that turns out to be rich enough to formulate and study the theoretical properties of a generalized version of the \textsc{ewa} forecaster, expanding the range of potential applications for online prediction strategies. 

From a technical pespective, the definition of our learning strategy, called the exponentially weighted barycentric (\textsc{ewb}) forecaster, relies on the central notion of barycenters of probability measures. Its analysis is then based on an adapted version of Jensen's inequality along with a general geometric property referred to as the measure contraction property that may be understood as a general definition of a lower Ricci curvature bound, in spaces that do not have a formal Riemannian structure.  

The use the measure contraction property in the Euclidean setting, can be traced back to \cite{BluKal99} that use the scale invariance of the Lebesgue measure to analyse the performance of the universal rebalanced porfolio algorithm introduced by
\cite{Cov91}. Later, the same argument was leveraged by \cite{Hazan06} to show that the \textsc{ewa} forecaster achieves logarithmic regret for general exponentially concave losses in Euclidean spaces. The argument we develop builds upon these original ideas, extending them to a more abstract framework.

The paper is organized as follows. Section \ref{sec:prelim} gathers basic definitions and necessary geometric background for our main results presented in Section \ref{sec:results}.
Proofs are reported to Section \ref{sec:proofs}. 
%------------------
%------------------
\section{Preliminaries}
\label{sec:prelim}
\subsection{Online optimization and the {\normalfont\textsc{ewa}} forecaster}
\label{subsec:oco}
We first recall the protocol of online optimization as well as the construction of the classical \textsc{ewa} forecaster and refer the reader to~\citet{CesLug06,Sha12} or~\citet{Haz16} for more details. 

Consider a decision space $M$ and a set $\mathcal L$ of loss functions $\ell:M\to \R$. At each round $t\ge 1$, a player has to choose to a point $x_t$ in $M$. After the player commits to their choice, the environment reveals a loss function $\ell_t\in\mathcal L$, the player incurs loss $\ell_t(x_t)$ and moves on to round $t+1$. The goal of the player is to minimize their cumulative loss over time. A traditional performance measure after $n$ rounds of the game is the regret $R_n$ that compares the cumulative loss of the player to the cumulative loss of the best fixed point in hindsight, i.e.,
\[R_n=\sum_{t=1}^n\ell_t(x_t)-\min_{x\in M}\sum_{t=1}^n\ell_t(x).\]

In this setting, the decision set $M$ is traditionally a convex subset of $\R^p$, for some $p\ge 1$, and a popular prediction strategy is given by the \textsc{ewa} forecaster defined as follows. First, let $m$ be a probability measure on the decision set $M$, encoding prior information\footnote{Whenever $M\subset\R^p$ is a bounded and convex set with non-empty interior, $m$ is traditionally the uniform distribution over $M$}. Then, given a sequence $(\beta_t)_{t\ge 1}$ of positive tuning parameters, the \textsc{ewa} forecaster is defined, for all $t\ge 1$, by
\begin{equation}
\label{eq:ewaclassic}
x_t:=\int_{M}x\,\mathrm{d}m_t(x),
\end{equation}
where $m_1:=m$ and, for all $t\ge 1$,
\begin{equation}
\label{eq:ewaclassicw}
\mathrm{d}m_{t+1}(x):=\frac{e^{-\beta_{t+1} \ell_{t}(x)}}{\int_{M}e^{-\beta_{t+1} \ell_{t}}\,\mathrm{d}m_t}\,\mathrm{d}m_t(x).
\end{equation}
This popular forecaster  naturally emphasizes the role of points that exhibit a small cumulative loss over time and its analysis is simplified by the convenient properties of the exponential function. Since its introduction by~\citet{Vov90} and~\citet{LitWar94}, the \textsc{ewa} forecaster has been analyzed in the Euclidean setting from many perspectives~\citep[see, e.g,][]{CesLug99,BluKal99,Hazan06}. The use of exponential weights has also found many applications in Statistics in the context of aggregation~\citep{Yan04,LeuBar06,Cat07,DalTsy07,DalTsy08,DalTsy09,JudRigTsy08,Alq08,Aud09,DalTsy12a,DalTsy12b}. 

In the Euclidean setting, the linear and convex structure of the decision space $M$ is important for several reasons. First, it allows to make sense of the integral defining $x_t$ in \eqref{eq:ewaclassic}. Second, it is essential to the traditional notion of convexity of losses $\ell\in\mathcal L$ usually invoked in the analysis of $x_t$. As it turns out, a simple adaptation of the construction of $x_t$ given in \eqref{eq:ewaclassic}, using the notion of barycenters, makes sense in an abstract metric space and reduces to the \textsc{ewa} forecaster in the Euclidean setting. The rest of this section reports some definitions and tools that will be used in the sequel to define and study this adaptation.

\subsection{Geodesic spaces and convexity}
Let $(M,d)$ be a metric space. For $\tau>0$, a path $\gamma:[0,\tau]\to M$ is called a geodesic if, for all $0\le s\le t\le \tau$,
\begin{equation}
\label{eq:scalegeod}
d(\gamma(s),\gamma(t))=\frac{t-s}{\tau}d(\gamma(0),\gamma(1)).
\end{equation}
Geodesics can be equivalently defined as constant-speed reparametrizations of length minimizing paths for an appropriate notion of length in metric spaces. 
\begin{defi}
\label{def:gspace}
The space $(M,d)$ is called geodesic if, for every $x,y\in M$, there exists a geodesic $\gamma:[0,1]\to M$ connecting $x$ to $y$, i.e., such that $\gamma(0)=x$ and $\gamma(1)=y$. 
\end{defi}

Fundamental examples of geodesic spaces are complete and connected Riemannian manifolds equipped with the Riemannian distance~\citep[Corollary 3.20]{BriHaf99}. The class of geodesic spaces includes however many other examples and we refer the reader to~\cite{BriHaf99,BurBurIva01,AleKapPet19_invitation,AleKapPet19} for more details. 

\begin{defi}
\label{def:convex}
Let $(M,d)$ be a geodesic space.
\begin{enumerate}[label=\rm{(\arabic*)}]
\item For $\alpha\in\R$, a function $f:M\to \R$ is called geodesically $\alpha$-convex if, for every geodesic $\gamma:[0,1]\to M$, the function 
\[t\in[0,1]\mapsto f(\gamma(t))-\frac{\alpha}{2} d^2(\gamma(0),\gamma(t)), \]
is convex. We call $f$ geodesically convex if it is geodesically $0$-convex and geodesically concave if $-f$ is geodesically convex.
\item For $\beta>0$, a function $f:M\to \R$ is called geodesically $\beta$-expconcave if the function $\exp(-\beta f)$ is geodesically concave.
\end{enumerate}
\end{defi}
% For simplicity, and when no confusion may arise, we'll simply say that a function $f:M\to\R$ is convex (resp. concave, $\alpha$-convex, exponentially $\beta$-concave) if it is geodesically convex (resp. concave, $\alpha$-convex, exponentially $\beta$-concave).
\begin{lem}
\label{lem:lcisec}
Let $(M,d)$ be a complete geodesic space and $f:M\to\R$ be a given function. 
\begin{enumerate}[label=\rm{(\arabic*)}]
\item Suppose that $f$ is geodesically $\beta$-expconcave for some $\beta>0$. Then, the function $f$ is geodesically convex.
\item Suppose that $f$ is geodesically $\alpha$-convex and $L$-Lipchitz for some $\alpha,L>0$. Then, it is geodesically $\beta$-expconcave for all 
$0< \beta\le \frac{\alpha}{L^2}$.
\end{enumerate}
\end{lem}
These definitions reduce to familiar notions of convexity in the context of Euclidean spaces. In particular, the second statement in Lemma \ref{lem:lcisec} extends to geodesic spaces a classical result by~\cite{Hazan06} in the context of Euclidean spaces. 

\subsection{Alexandrov curvature bounds}
\label{subsec:alex}
For $\kappa\in \R$, a remarkable geodesic space is the $\kappa$-plane  $(M^2_{\kappa},d_{\kappa})$ defined as the unique\footnote{This space corresponds to the hyperbolic plane with curvature $\kappa$ for $\kappa<0$, the Euclidean plane for $\kappa=0$ and the $2$-dimensional unit sphere with angular metric multiplied by $1/\sqrt{\kappa}$ for $\kappa>0$.} (up to isometry) $2$-dimensional, complete and simply connected, Riemannian manifold with constant sectional curvature $\kappa$, equipped with its Riemannian distance $d_{\kappa}$. The diameter $D_{\kappa}$ of $M^{2}_{\kappa}$ is 
\[D_{\kappa}:=\left\{\begin{array}{cc}
     +\infty&\mbox{if}\quad\kappa\le0,\\
     \pi/\sqrt{\kappa}&\mbox{if}\quad\kappa>0,
\end{array}\right.\]
and there is a unique geodesic $[0,1]\to M$ connecting $x$ to $y$ in $M^2_{\kappa}$ provided $d_{\kappa}(x,y)<D_{\kappa}$. 

Given a metric space $(M,d)$, we call triangle in $M$ any set of three points $\{p,x,y\}\subset M$. For $\kappa\in\R$, a comparison triangle for $\{p,x,y\}\subset M$ in $M^2_{\kappa}$ is an isometric copy $\{\bar p,\bar x,\bar y\}\subset M^2_{\kappa}$ of $\{p,x,y\}$ in $M^2_{\kappa}$ (i.e., pairwise distances are preserved). Such a comparison triangle always exists and is unique (up to an isometry) provided the perimeter $\mathrm{peri}\{p,x,y\}:=d(p,x)+d(p,y)+d(x,y)< 2D_{\kappa}$. 

\begin{defi}
\label{def:curvb}
For $\kappa\in \R$, we say that a geodesic space $(M,d)$ has curvature bounded below by $\kappa$, and denote ${\rm curv}(M)\ge\kappa$, if for any triangle $\{p,x,y\}\subset M$ with $\mathrm{peri}\{p,x,y\}< 2D_{\kappa}$ and any geodesic $\gamma:[0,1]\to M$ connecting $x$ to $y$ in $M$, we have
\begin{equation}
\label{eq:curvcomp}
\forall t\in[0,1],\quad d(p,\gamma(t))\ge d_{\kappa}( \bar p,\bar\gamma(t)),
\end{equation}
where $\{\bar p,\bar x,\bar y\}$ is the unique comparison triangle of $\{p,x,y\}$ in $M^2_{\kappa}$ and where $\bar\gamma:[0,1]\to M^2_{\kappa}$ is any geodesic connecting $\bar x$ to $\bar y$ in $M^2_{\kappa}$. Similarly, we say that $M$ has curvature bounded above by $\kappa$, and denote ${\rm curv}(M)\le\kappa$, if the same holds with opposite inequality in \eqref{eq:curvcomp}.
\end{defi}
The $0$-plane $(M^2_0,d_0)$ is the familiar Euclidean plane. The properties of the Euclidean plane allow to reformulate, in simpler terms, the definition of curvature bounds in the case $\kappa=0$.
\begin{cor}
\label{cor:0cb}
A geodesic space $(M,d)$ satisfies ${\rm curv}(M)\ge0$ iff, for all $p,x,y\in M$, any geodesic $\gamma:[0,1]\to M$ connecting $x$ to $y$ in $M$, and any $t\in [0,1]$,
\begin{equation}
\label{eq:curvcomp2}
d^2(p,\gamma(t))\ge (1-t)d^2(p,x)+td^2(p,y)-t(1-t)d^2(x,y).
\end{equation}
Similarly, ${\rm curv}(M)\le0$ if the same holds with opposite inequality in \eqref{eq:curvcomp2}.
\end{cor}
Combining Definition \ref{def:convex} and Corollary \ref{cor:0cb}, it follows that ${\rm curv}(M)\le 0$ iff, for all $p\in M$, the function $d^2(p,.):M\to \R_+$ is $2$-convex. In particular, we deduce directly from Lemma \ref{lem:lcisec}, and the triangular inequality, the following fact.

\begin{cor}
\label{cor:d2isbetaconcave}
Let $(M,d)$ be a geodesic space satisfying ${\rm curv}(M)\le 0$ and with finite diameter. Then, for all $p\in M$, the function $d^2(p,.):M\to \R_+$ is $\beta$-concave for all $0<\beta\le 1/( 2\,\mathrm{diam}(M)^2)$.
\end{cor}
Next is a list of (complete and separable\footnote{These properties seem to be necessary for the validity of Jensen's inequality as explained below.}) geodesic spaces that have curvature bounds in the sense of Definition \ref{def:curvb}.

\begin{exm}
\label{exm:curv}
\normalfont
\begin{enumerate}[label=\rm{(\arabic*)}]
\item A normed vector space $V$ has a curvature bound from above, or below, iff it is a pre-Hilbert space\footnote{This follows by combining Proposition 4.5 in~\citet{BriHaf99} and the fact that angles are well defined for geodesic spaces with upper or lower bounded curvature.}, in which case it satisfies ${\rm curv}(V)\ge 0$ and ${\rm curv}(V)\le 0$ since \eqref{eq:curvcomp2} holds as an identity. 
\item For any $\kappa\in \R$, a complete Riemannian manifold $M$ with sectional curvature everywhere lower bounded by $\kappa$ satisfies ${\rm curv}(M)\ge\kappa$.     
\item For $\kappa\le 0$, a complete and simply connected Riemannian manifold $M$ with sectional curvature everywhere upper bounded by $\kappa$ satisfies ${\rm curv}(M)\le\kappa$.
\item The frontier $\partial K$ of a convex body\footnote{i.e., a convex and compact subset with non-empty interior.} $K\subset \R^p$ equipped with its length metric\footnote{Here we mean the length metric inherited from the Euclidean distance. Roughly speaking, this means that the distance between $x,y\in\partial K$ is defined as the length of the shortest continuous (in terms of the Euclidean topology) path connecting $x$ to $y$ and whose image is included in $\partial K$.} is a geodesic space satisfying $\mathrm{curv}(\partial K)\ge 0$.
\item A (complete and separable) geodesic space $(\Omega,d)$ satisfies $\mathrm{curv}(\Omega)\ge 0$ iff the space $\mathcal P_2(\Omega)$, equipped with the $2$-Wasserstein metric, satisfies $\mathrm{curv}(\mathcal P_2(\Omega))\ge 0$~\cite[Proposition 2.10]{Stu06I}.
\end{enumerate}
\end{exm}

\subsection{Barycenters and Jensen's inequality}
\label{subsec:baretjen}
Given a metric space $(M,d)$, let $\mathcal P_2(M)$ be the set of Borel probability measures $m$ on $M$ satisfying, for all $x\in M$,
\[\mathcal V_m(x):=\int_M d^2(x,y)\,{\rm d}m(y)<+\infty.\]
For $m\in\mathcal P_2(M)$, we call $\mathcal V_m:M\to\R_+$ its variance functional and denote
\begin{equation}
\label{eq:variance}
\mathcal V^*_m:=\inf_{x\in M}\mathcal V_m(x).
\end{equation}

\begin{defi}
Given a metric space $(M,d)$ and $m\in\mathcal P_2(M)$, a barycenter of $m$ is any $x^*\in M$ such that
\begin{equation}
\label{xstarintrobary}
\mathcal V_m(x^*)=\mathcal V^*_m.
\end{equation}
\end{defi}
Barycenters provide a generalization\footnote{Note for instance that if $(M,d)=(\R^p,\|.-.\|_2)$ and if $\int\|.\|^2_2\,\mathrm{d}m<+\infty$, then
$x^*=\int x\,\mathrm{d}m(x)$ is the unique minimizer of $x\in M\mapsto\int \|x-.\|^2_2\,\mathrm{d}m$.} of the notion of mean value when $M$ has no linear structure. While alternative notions of mean value in a metric space have been proposed, barycenters are often favored for their simple interpretation and constructive definition as solution of an optimization problem. The question of existence and uniqueness of barycenters has been addressed in a number of settings. While uniqueness will be of less interest in the sequel, we mention two classical results on the existence of barycenters. 
\begin{itemize}
    \item If $M$ is geodesic and locally compact, then any $m\in \mathcal P_2(M)$ admits at least one barycenter\footnote{The Hopf-Rinow Theorem~\citep[Proposition 3.7]{BriHaf99} states that a closed and bounded subset of a locally compact geodesic space is compact. Since the variance functional $\mathcal V_m$ is lower semi-continuous, by application of Fatou's Lemma, the statement follows from a standard compactness argument.}. 
\item If $M$ is geodesic, complete and satisfies ${\rm curv}(M)\le 0$, then any $m\in \mathcal P_2(M)$ admits a (unique) barycenter~\citep[Theorem 4.9]{Stu03}. 
\end{itemize}

In specific geodesic spaces which do not necessarily satisfy these assumptions, the existence and uniqueness of barycenters can be obtained via a taylored analysis such as in Wasserstein spaces~\citep[see, e.g.][]{AguCar11,GouLou17}. 

Along with the notion of barycenters, a fundamental result needed next is a suitable version of Jensen's inequality. We mention two results in this direction.
\begin{lem}[\citealp{Stu03}, Theorem 6.2]
\label{lem:jensen1}
Let $(M,d)$ be a complete geodesic space with ${\rm curv}(M)\le 0$. Let $m\in \mathcal P_2(M)$ and let $x^*$ be its unique barycenter. Let $f:M\to \R$ be convex and lower semi-continuous. Then we have
    \[f(x^*)\le\int_M f\,\mathrm{d}m,\]
    provided $f$ is either positive or in $L^1(m)$.
\end{lem}

\begin{lem}[\citealp{Par20}, Theorem 1.1]
\label{lem:jensen2}
Let $(M,d)$ be a complete and separable geodesic space with ${\rm curv}(M)\ge \kappa$, for some $\kappa\in \R$. Let $m\in \mathcal P_2(M)$ and suppose that it admits at least one barycenter $x^*$. Let $f:M\to \R$ be $\alpha$-convex, for some $\alpha\in \R$, and Lipschitz in a neighborhood of $x^*$. Then we have
    \[f(x^*)\le\int_M f\,\mathrm{d}m-\frac{\alpha}{2}\mathcal V^*_{m},\]
    provided $f$ is either positive or in $L^1(m)$.
\end{lem}

\subsection{The measure contraction property}
Curvature bounds in the sense of Alexandrov, defined in paragraph \ref{subsec:alex}, generalize the notion of sectional curvature bounds of Riemannian manifolds (see statements $(2)$ and $(3)$ of Example \ref{exm:curv}). In this paragraph, we discuss the measure contraction property (\textsc{mcp}), which generalizes lower Ricci curvature bounds for Riemannian manifolds. In the recent years, a number of synthetic definitions of lower Ricci curvature bounds have been developed in the abstract context of metric-measure spaces~\citep[see, e.g.,][]{Stu06I,Stu06II,Oht07,LotVil09}. These have shown to imply many of the analytical properties expected in the context of Riemannian manifolds, under a Ricci curvature lower bound, in a wider context. The \textsc{mcp} property defined below is known to be one of the weakest forms of such definitions (see Remark \ref{rem:mcpandcd}). 

Consider a metric-measure space $(M,d,m)$ where $(M,d)$ is a complete and separable geodesic space and $m$ is a reference positive Borel measure on $M$. Suppose in addition that  $0<m(B(x,r))<+\infty$ for all $x\in M$ and all $r>0$, where $B(x,r):=\{y\in M:d(x,y)<r\}$. 

\begin{defi}
 Given $x\in M$, we call geodesic homothety of center $x$ any measurable map $h_x:M\times[0,1]\to M$ such that, for $m$-a.e. $y\in M$, the map $h_x(y,.):[0,1]\to M$ is a geodesic connecting $x$ to $y$. 
\end{defi}
A geodesic homothety is to be understood as a generalization of the map $h_x(y,\e):=(1-\e)x+\e y$, in a Euclidean space, contracting points towards the center $x$ with ratio $\e$. The \textsc{mcp} property defined next, quantifies the way Borel subsets of $M$ contract toward the center $x$ under the the action of $h_x$. For $\kappa\in\R$ and $r\ge 0$, we denote 
\begin{equation}
\label{eq:skappa}
s_{\kappa}(r):=\left\{
\begin{array}{cc}
     \frac{\sin(r\sqrt{\kappa})}{\sqrt{\kappa}}&\mbox{if }\kappa>0,\\
     r &\mbox{if }\kappa=0,\\
     \frac{\sinh(r\sqrt{-\kappa})}{\sqrt{-\kappa}}&\mbox{if }\kappa<0.
\end{array}\right.
\end{equation}

\begin{defi}[\citealp{Oht07}, Lemma 2.3]
For $\kappa\in\R$ and $p>1$, the space $(M,d,m)$ is said to satisfy the measure contraction property \textsc{mcp}$(\kappa,p)$ if, for all $x\in M$, there exists a geodesic homothety $h_x$ with center $x$ such that, for every measurable subset $A\subset M$ (with $A\subset B(x,\pi\sqrt{(p-1)/\kappa})$ if $\kappa>0$) and every $\e\in[0,1]$,
\[m(A^{\e}_{x})\ge \int_{A}\e\left(\frac{s_{\kappa}(\e d(x,y)/\sqrt{p-1})}{s_{\kappa}(d(x,y)/\sqrt{p-1})}\right)^{p-1}\,m(\mathrm{d}y),\]
where $A^{\e}_{x}:=\{h_{x}(y,\e):y\in A\}$ and with the convention $0/0=1$.
\end{defi}

\begin{exm}
\label{exm:mcp}
\normalfont
\begin{enumerate}[label=\rm{(\arabic*)}]
\item Suppose $M$ is a complete Riemannian manifold of dimension $p$. Let $d$ be the Riemannian distance and $m$ the volume measure. Then $\mathrm{Ric}_M\ge \kappa$ iff $(M,d,m)$ satisfies the \textsc{mcp}$(\kappa,p)$ property. In addition, for any function $f:M\to\R_+$ and any $q>0$ such that $f^{1/q}$ is geodesically concave, the weighted space $(M,d,fm)$ satisfies the \textsc{mcp}$(\kappa,p+q)$ property.
~\citep[Corollary 5.5(i)]{Stu06II}.
\item Let $(M,d)$ be a complete and locally compact geodesic space with curvature lower bounded by $\kappa\in\R$ in the sense of Definition \ref{def:curvb} and finite Hausdorff dimension $p$. Then letting $m$ be the $p$-dimensional Hausdorff measure, $(M,d,m)$ satisfies the \textsc{mcp}$((p-1)\kappa,p)$ property~\citep[Theorem 1.1]{KuwShi10}.
\end{enumerate}
\end{exm}

\begin{rem}
\label{rem:renormalize}
\normalfont
Given a metric measure space $(M,d,m)$ and $\alpha>0$, it appears clearly from the definition that $(M,d,m)$ satisfies the \textsc{mcp}$(\kappa,p)$ property if and only $(M,d,\alpha m)$ satisfies the same property. In particular, provided $0<m(M)<+\infty$, we can suppose that $m$ is a probability measure without loss of generality.
\end{rem}

\begin{rem}
\label{rem:mcpandcd}
\normalfont
An alternative, and more popular, synthetic definition of Ricci curvature lower bound is the curvature-dimension condition \textsc{cd}$(\kappa,p)$. Under minimal regularity conditions on the space $(M,d,m)$, this condition is known to imply the \textsc{mcp}$(\kappa,p)$ property~\citep[Theorem 5.4]{Stu06II}. Conversely, some examples of spaces satisfying the \textsc{mcp} property but not the \textsc{cd} property are known~\citep[see, e.g.,][]{Jui09,Riz16}.
\end{rem}
%------------------
%------------------
\section{Results}
\label{sec:results}

We adopt the same notation as in paragraph \ref{subsec:oco}. We suppose that the decision set $(M,d)$ is a (complete and separable) geodesic space. We fix a prior distribution $m\in\mathcal P_2(M)$ and suppose that the following properties hold.
\begin{enumerate}[label=(A\arabic*),leftmargin=*]
\item\label{A1} \textbf{Existence of barycenters}\\ 
Any $\mu\in \mathcal P_2(M)$ admits at least one barycenter.
\item\label{A2} \textbf{Jensen's inequality}\\ 
For any $\mu\in \mathcal P_2(M)$, any barycenter $x^*$ of $\mu$ and any geodesically convex $f:M\to \R$, either positive or in $L^1(\mu)$, we have 
    \[f(x^*)\le\int_{M}f\,{\rm d}\mu.\]
\item\label{A3} \textbf{Measure contraction property}\\ 
There exists $\kappa\in\R$ and $p>1$ such that $(M,d,m)$ satisfies the \textsc{mcp}$(\kappa,p)$ property.
\end{enumerate}

Note that Assumptions \ref{A1} and \ref{A2} refer to characteristics of the metric space $(M,d)$ while \ref{A3} refers to a property of the metric-measure space $(M,d,m)$. It follows from section \ref{sec:prelim} that these properties are satisfied in a wide setting. For instance, typical examples for which all three assumptions are satisfied at once include: 
\begin{itemize}
    \item Bounded and locally compact geodesic spaces $(M,d)$ with finite Hausdorff dimension $p>1$, curvature lower bounded by $\kappa/(p-1)$ for some $\kappa\in\R$ in the sense of Definition \ref{def:curvb} and equipped with the renormalized $p$-dimensional Hausdorff measure $m$. 
    \item Complete and connected Riemannian manifolds of dimension $d<p$, with the Riemannian distance, sectional curvature lower bounded by $\kappa/(p-1)$ for some $\kappa\in\R$ (and hence $\mathrm{Ric}_M\ge \kappa$) and reference measure $m$ with density $f$ with respect to the volume measure such that $f^{1/(p-d)}$ is geodesically concave. 
\end{itemize}

\subsection{The {\normalfont\textsc{ewb}} forecaster}
\label{subsec:ewb}
We are now in position to define a learning strategy we call the Exponentially Weighted Barycentric (\textsc{ewb}) forecaster.
\begin{defi}
\label{defi:ewb}
Let $(\beta_t)_{t\ge 1}$ be a sequence of positive tuning parameters. Then, for $t\ge 1$, we define $x_t$ as a barycenter of $m_t$, i.e.,
\begin{equation}
\label{eq:ewametric}
x_t\in\underset{x\in  M}{\arg\min}\int_{M}d^2(x,y)\,m_t(\mathrm{d}y),
\end{equation}
where, as in the classical setting, $m_1:=m$ and, for all $t\ge 1$,
\begin{equation}
\nonumber
\mathrm{d}m_{t+1}(x):=\frac{e^{-\beta_{t+1} \ell_{t}(x)}}{\int_{M}e^{-\beta_{t+1} \ell_{t}}\,\mathrm{d}m_t}\,\mathrm{d}m_t(x).
\end{equation}
\end{defi}

\subsection{Regret bounds}
\label{subsec:regret}
Throughout the rest of the section, we denote \[\psi(r):=(r\coth r)\exp(-r\coth r),\]
for $r>0$, set $\psi(0)=e^{-1}$ and define, for all $x\in M$, all $p>1$ and all $\kappa\in\R$,
\[c_{\kappa,p}(x):=\left\{\begin{array}{cc}
    1 &\mbox{if } \kappa\ge 0,  \\
    \int_M \psi\left(d(x,y)\sqrt{\frac{-\kappa}{p-1}}\right)\,m(\mathrm{d}y) & \mbox{if } \kappa< 0.
\end{array}
\right.\]

\begin{wrapfigure}{r}{0.5\textwidth}
  \begin{center}
   \includegraphics[width=0.48\textwidth]{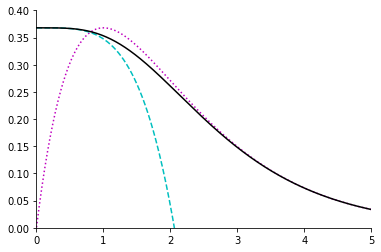}
     \caption{Plot of function $\psi(r)=(r\coth r)\exp(-r\coth r)$ (solid line), equivalent to $\psi_0(r)=1/e-r^4/18e$ (dashed line) at $0$ and to $\psi_{\infty}(r)=r\exp(-r)$ (dotted line) at $+\infty$.}
  \label{fig:psi}
  \end{center}
\end{wrapfigure}

\begin{thm}
\label{thm:ewb}
Assume that \ref{A1}, \ref{A2} and \ref{A3} hold. Suppose that there exists $\beta>0$ such that every $\ell\in\mathcal L$ is geodesically $\beta$-expconcave. Then taking $\beta_t=\beta$ for all $t\ge 1$ and selecting any
\begin{equation}
    \label{eq:xstarn}
x^*_n\in\underset{x\in M}{\arg\min}\,\sum_{t=1}^n\ell_t(x),
\end{equation}
the regret of the \textsc{ewb} forecaster satisfies, for all $n\ge 2$, 
\[R_n\le \frac{1}{\beta}\left(2+\ln\frac{1}{c_{\kappa,p}(x^*_n)}\right)+\frac{p\ln n}{\beta}.\]
\end{thm}

When $\kappa\ge0$, the regret bound displayed in Theorem \ref{thm:ewb} reads
\[R_n\le \frac{2}{\beta}+\frac{p\ln n}{\beta}.\]
This regret bound was already obtained for the case where $M\subset \R^p$ is a bounded convex set, with non-empty interior, equipped with Euclidean metric and uniform measure~\citep[Theorem 7]{Hazan06}. The above result shows that this behavior is exactly preserved whenever $(M,d,m)$ satisfies Assumptions \ref{A1}-\ref{A3} with $\kappa\ge 0$. In the case of $\kappa<0$, the regret bound displays the additional term
\[\frac{1}{\beta}\ln\frac{1}{c_{\kappa,p}(x^*_n)}=\frac{1}{\beta}\ln\frac{1}{\int_M \psi\left(d(x^*_n,y)\sqrt{\frac{-\kappa}{p-1}}\right)\,m(\mathrm{d}y)},\]
whose behavior, as $n$ goes to infinity, isn't straightforward in full generality. Under the additional assumption that $M$ has bounded diameter, the monotonicity of function $\psi$, displayed in Figure \ref{fig:psi}, implies however that this additional term is upper bounded, independently of $n$, by
\[\frac{1}{\beta}\ln\frac{1}{\psi\left(\mathrm{diam}(M)\sqrt{\frac{-\kappa}{p-1}}\right)},\]
guaranteeing in this case also a regret of order at most $(p\ln n)/\beta$. In the case where $\kappa<0$ and $M$ isn't bounded, an obvious restriction on $(M,d,m)$, allowing for a logarithmic regret, is to impose that
\begin{equation}
\label{eq:infc}
\inf_{x\in M}c_{\kappa,p}(x)=\inf_{x\in M}\int_M \psi\left(d(x,y)\sqrt{\frac{-\kappa}{p-1}}\right)\,m(\mathrm{d}y)>0.
\end{equation}
Finally we note that, as in the Euclidean setting, an important aspect of the above regret bound is that it does not require the losses to be Lipschitz. As a remark, we show however that, at the price of this additional requirement, we can obtain a similar regret bound under a more general assumption.

\begin{pro}
\label{pro:ewb2}
Assume that \ref{A1} and \ref{A2} hold. Assume that there exists constants $p,r_0>0$ and a function $c:M\to\R_+$ such that
\begin{equation}
\label{eq:mesballs}
\forall x\in M,\forall r\in(0,r_0],\quad  m(B(x,r))>c(x)r^p.
\end{equation}
Assume finally that there exist $\beta,L>0$ such that every $\ell\in\mathcal L$ is geodesically $\beta$-expconcave and $L$-Lipschitz. Then taking $\beta_t=\beta$ for all $t\ge 1$, the regret of the \textsc{ewb} forecaster satisfies, for all $n\ge 1/r_0$,
\[R_n\le L+\frac{1}{\beta}\ln\frac{1}{c(x^*_n)}+\frac{p\ln n}{\beta},\]
for $x^*_n\in M$ as in \eqref{eq:xstarn}. 
\end{pro}

We show next that property \eqref{eq:mesballs} is indeed more general than the measure contraction property.

\begin{pro}
\label{pro:mcplessgen}
Suppose that $(M,d,m)$ satisfies the measure contraction property \textsc{mcp}$(\kappa,p)$ for some $\kappa\in \R$ and some $p>1$. Then, there exists $p,r_0>0$ and $c:M\to\R_+$ such that property \eqref{eq:mesballs} holds.
\end{pro}

We end the paragraph by mentioning results valid for bounded losses that are only geodesically convex. For brevity, the proof is reported in the Appendix.

\begin{thm}
\label{thm:ewb3}
Assume that \ref{A1} and \ref{A2} hold. Suppose that every $\ell\in\mathcal L$ is geodesically convex and $[a,b]$-valued. Then setting, for all $t\ge 1$,
\[\beta_t=\frac{2c_1}{(b-a)}\sqrt{\frac{\ln(t\lor 2)}{t}}\quad\mbox{where}\quad c_1:=\left(\frac{3}{2}\right)^{\frac{1}{4}}\le 1.1,\] the regret of the \textsc{ewb} forecaster satisfies the following.
\begin{enumerate}[label=\rm{(\arabic*)}]
\item If Assumption {\normalfont\ref{A3}} holds, then for all $n\ge 2$,
\[\frac{R_n}{b-a}\le 1+c_1\left(1+\ln\frac{1}{c_{\kappa,p}(x^{\star}_n)}\right)\sqrt{pn\ln n},\]
where $c_{\kappa,p}$ and $x^*_n$ are as in Theorem \ref{thm:ewb}.
\item If more generally property \eqref{eq:mesballs} holds and all losses $\ell\in\mathcal L$ are also $L$-Lipschitz, then for all $n\ge 1/r_0$, 
\[\frac{R_n}{b-a}\le \frac{L}{b-a}+c_1\left(1+\ln\frac{1}{c(x^*_n)}\right)\sqrt{pn\ln n},\]
for $x^*_n\in M$ as in \eqref{eq:xstarn}.
\end{enumerate}
\end{thm}

\subsection{Online to batch conversion}
\label{subsec:otbc}
The principle of online-to-batch conversion is a classical way to exploit algorithms, developed for sequential prediction, in the context of statistical learning. This section presents a simple adaptation of this procedure in the context of metric spaces.

Consider the following statistical learning problem. Let $(M,d)$ be a metric space, let $\mathcal Z$ be an arbitrary measurable space and let $\ell:M\times \mathcal Z\to \R$ be a fixed loss function. Suppose given a collection (or batch) $\{Z_{i}\}_{i=1}^n$ of independent and identically distributed $\mathcal Z$-valued random variables with same distribution as (and independent from) a generic random variable $Z$. Finally, consider the task of constructing $\hat\theta_n\in M$ based on $\{Z_{i}\}_{i=1}^n$ and such that the excess risk
\[\esp[\ell(\hat\theta_n,Z)]-\inf_{\theta\in M}\esp[\ell(\theta,Z)],\]
is as small as possible.

To that aim, consider first the problem of online optimization studied so far with decision space $M$ and loss functions $\mathcal L=\{\ell(.,z):z\in\mathcal Z\}$. In this setting, at each round $t\ge 1$ the player first chooses a point $\theta_t\in M$, the environment then reveals a point $z_t$, the player incurs loss $\ell(\theta_t,z_t)$ and moves on to round $t+1$. Here, the players strategy $(\theta_t)_{t\ge 1}$ can be formally associated to a sequence of measurable maps $(\vartheta_t)_{t\ge 1}$ where $\vartheta_1$ is constantly equal to $\theta_1$ and, for $t\ge 2$, $\vartheta_t:\mathcal{Z}^{t-1}\to M$ is such that
\[\theta_t=\vartheta_t(z_1,\dots,z_{t-1}).\]

Now, suppose given such a prediction strategy $(\vartheta_t)_{t\ge 1}$ and suppose that, for all $n\ge 1$, there exists $B_n>0$ satisfying 
\begin{equation}
\label{eq:boundforotbc}
\sum_{t=1}^n\ell( \theta_t,z_t)-\inf_{\theta\in M}\sum_{t=1}^n\ell(\theta,z_t)\le B_n,
\end{equation}
uniformly over the outcome sequence $(z_1,\dots,z_n)\in\mathcal Z^n$. 

Then, coming back to the statistical learning problem, consider $\hat\theta_n$ to be a barycenter of the (random) probability measure 
\[\frac{1}{n+1}\delta_{\vartheta_1}+\frac{1}{n+1}\sum_{t=2}^{n+1}\delta_{\vartheta_t(Z_1,\dots,Z_{t-1})},\]
on $M$, i.e.,
\begin{equation}
    \label{eq:defthetan}
\hat\theta_n\in\underset{\theta\in M}{\arg\min}\left\{d^2(\theta,\vartheta_1)+\sum_{t=2}^{n+1} d^2(\theta,\vartheta_t(Z_1,\dots,Z_{t-1}))\right\}.
\end{equation}
Then we have the following result.
\begin{thm}
\label{thm:otbc}
Suppose that the metric space $(M,d)$ satisfies Assumptions \ref{A1} and \ref{A2}. Suppose that \eqref{eq:boundforotbc} holds uniformly over the outcome sequence $(z_1,\dots,z_n)\in\mathcal Z^n$ and that, for all $z\in\mathcal Z$, the function $\ell(.,z):M\to \R$ is geodesically convex. 
Then, for all $n\ge 1$,
\[\esp[\ell(\hat\theta_n,Z)]-\inf_{\theta\in M}\esp[\ell(\theta,Z)]\le \frac{B_{n+1}}{n+1}.\]
\end{thm}
Combining Theorems \ref{thm:ewb} and \ref{thm:otbc}, we readily obtain the following result. A similar adaptation of Theorem \ref{thm:ewb3} is left to the reader.

\begin{cor}
\label{cor:otbc2}
Suppose that $(M,d)$ satisfies Assumptions \ref{A1} and \ref{A2}. Let $m$ be a prior distribution over $M$ such that $(M,d,m)$ satisfies Assumption \ref{A3} and condition \eqref{eq:infc} if $\kappa<0$. Suppose that there exists $\beta>0$ such that, for all $z\in\mathcal Z$, the function $\ell(.,z):M\to \R$ is geodesically $\beta$-expconcave. Let $\hat\theta_n$ be as in \eqref{eq:defthetan} where $\vartheta_1$ is a barycenter of $m$ and, for $2\le t\le n+1$, $\vartheta_t(Z_1,\dots,Z_{t-1})$ is a barycenter of the (random) probability measure $m_{t}$ defined by 
\[\mathrm{d}m_{t}(\theta)=\frac{e^{-\beta\ell(\theta,Z_{t-1})}}{\int_M e^{-\beta\ell(.,Z_{t-1})}\mathrm{d}m_{t-1}}\mathrm{d}m_{t-1}(\theta),\]
where $m_1:=m$. Then, for all $n\ge 2$,
\[\esp[\ell(\hat\theta_n,Z)]-\inf_{\theta\in M}\esp[\ell(\theta,Z)]\le \frac{2+\ln(\inf_M c_{\kappa,p})^{-1}}{\beta (n+1)}+\frac{p\ln (n+1)}{\beta (n+1)}.\]
\end{cor}

\subsection{Example: The $2$-Wasserstein space over $\R^d$}
Consider the set $\mathcal P_2(\R^d)$ of all square integrable (Borel) probability measures over $\R^{d}$. 
Given $\mu,\nu\in\mathcal P_2(\R^d)$, denote $\Pi_{\mu,\nu}$ the set of couplings between $\mu$ and $\nu$, i.e., the set of probability measures $\pi$ on $\R^d\times\R^d$ satisfying $\pi(.\times\R^d)=\mu$ and $\pi(\R^d\times .)=\nu$. The $2$-Wasserstein metric $W_2$ on $\mathcal P_2(\R^d)$ is defined by 
 \[W^2_2(\mu,\nu):=\inf_{\pi\in\Pi_{\mu,\nu}}\int_{\R^d\times\R^d}\|x-y\|^2_2\,{\rm d}\pi(x,y).\]
The metric space $(\mathcal P_2(\R^d),W_2)$ is called the $2$-Wasserstein space over $\R^d$. It is  complete, separable, geodesic and has non-negative curvature in the sense of Definition \ref{def:curvb}~\citep[Proposition 2.10]{Stu06I}. The existence and uniqueness of barycenters of probability measures on $\mathcal P_2(\R^d)$ have been established under general conditions. For instance~\citet{GouLou17} show that, for any $P\in\mathcal P_2(\mathcal P_2(\R^d))$,
\begin{itemize}
    \item $P$ admits a barycenter,
    \item if there exists a Borel subset $B\subset \mathcal P_2(\R^d)$ such that $P(B)>0$ and such that every $\mu\in B$ has a density with respect to the Lebesgue measure, then this barycenter is unique.
 \end{itemize}
Assumptions \ref{A1}-\ref{A3} are therefore verified for a subset $M\subset \mathcal P_2(\R^d)$ equipped with the Wasserstein metric provided the following conditions are satisfied: 
\begin{enumerate}[label=\rm{(\arabic*)}]
\item $M$ is a geodesically convex subset of $\mathcal P_2(\R^d)$ in the sense that for any $\mu,\nu\in M$ and any geodesic $\gamma:[0,1]\to \mathcal P_2(\R^d)$ connecting $\mu$ to $\nu$, the image of $\gamma$ is included in $M$,
\item $M$ has finite Hausdorff dimension $p$
\end{enumerate}

We end with a result of special interest for online learning with the $W_2$ loss in the $2$-Wasserstein space. Recall first that, according to the Kantorovich dual representation of $W_2$, for all $\mu,\nu\in\mathcal P_2(\R^D)$,
\[W^2_2(\mu,\nu)=\sup_{\varphi}\int(\|.\|^2-2\varphi)\,{\rm d}\mu+\int(\|.\|^2-2\varphi^*)\,{\rm d}\nu,\]
where $\varphi:\R^D\to\R\cup\{-\infty\}$ ranges over the set of proper and lower semi-continuous convex functions and $\varphi^*$ denotes the Fenchel-Legendre conjugate of $\varphi$. In addition, the supremum is always attained and a $\varphi_{\mu\to\nu}$ achieving the max is called an optimal Kantorovich potential.   
\begin{lem}
\label{lem:w2mixable}
Let $P\in\mathcal P_2(\mathcal P_2(\R^D))$ and suppose $\Delta:=\mathrm{diam}(\mathrm{supp}(P))<+\infty$. Let $\mu^*$ be a barycenter of $P$ and suppose that, for $P$-almost all $\nu$, $\varphi_{\mu^*\to\nu}$ is $\alpha(\nu)$-strongly convex for a measurable $\alpha:\mathcal P_2(\R^D)\to \R_+$. Then $\mu^*$ is unique and, for all $\mu$ in the support of $P$, and all \[0<\beta\le\frac{8\mathcal V^*_P}{\Delta^4},\]
we have
\begin{equation}
\label{eq:vi}
W^2_2(\mu^*,\mu)\le -\frac{C^{-1}_{\rm var}}{\beta}\ln\int\exp(-\beta W^2_2(\mu,.))\,\mathrm{d}P,
\end{equation}
where $C_{\rm var}:=\int \alpha(\nu){\rm d}P(\nu)$.
\end{lem}

%------------------
%------------------
\section{Proofs}
\label{sec:proofs}

 \subsection{Proof of Lemma \ref{lem:lcisec}} 
$(1)$ Suppose that $f$ is $\beta$-expconcave for some $\beta>0$. Then, since the logarithm is increasing and concave, we get that, for any geodesic $\gamma:[0,1]\to M$ and all $t\in[0,1]$, 
\begin{align*}
f(\gamma_t)&=-\frac{1}{\beta}\ln e^{-\beta f(\gamma_t)}\\
&\le -\frac{1}{\beta}\ln ((1-t)e^{-\beta f(\gamma_0)}+t e^{-\beta f(\gamma_1)})\\
&\le (1-t)f(\gamma_0)+tf(\gamma_1).
\end{align*}

\noindent$(2)$ Suppose $f:M\to\R$ is geodesically $\alpha$-convex and $L$-Lipchitz for some $\alpha,L>0$. Using the fact that $M$ is complete, it is enough to show that for every $0\le \beta\le \alpha/L^2$ and every geodesic $\gamma:[0,1]\to M$, we have 
\[\frac{1}{2}e^{-\beta f(\gamma_0)}+\frac{1}{2}e^{-\beta f(\gamma_1)}\le e^{-\beta f(\gamma_{1/2})},\]
or, equivalently, that 
\begin{equation}
\label{lem:lcisec:e1}
\frac{1}{2}e^{\beta(f(z)-f(x))}+\frac{1}{2}e^{\beta(f(z)-f(y))}\le 1,
\end{equation}
where $x=\gamma_0,y=\gamma_1$ and $z=\gamma_{1/2}$. By $\alpha$-convexity of $f$, we have that 
\[f(z)\le\frac{1}{2}f(x)+\frac{1}{2}f(y)-\frac{\alpha}{8}d(x,y)^2,\]
so that
\[\beta(f(z)-f(x))\le \frac{\beta}{2}(f(y)-f(x))-\frac{\alpha\beta}{8}d(x,y)^2,\] 
and
\[\beta(f(z)-f(y))\le \frac{\beta}{2}(f(x)-f(y))-\frac{\alpha\beta}{8}d(x,y)^2.\]
For \eqref{lem:lcisec:e1} to be satisfied, it is therefore enough to have 
\begin{equation}
\nonumber
%\label{lem:lcisec:e2}
\frac{1}{2}e^{\frac{\beta}{2}(f(y)-f(x))}+\frac{1}{2}e^{\frac{\beta}{2}(f(x)-f(y))}=\cosh(\frac{\beta}{2}(f(x)-f(y)))\le e^{\frac{\alpha\beta}{8}d(x,y)^2}.
\end{equation}
Now, using the fact that $\cosh(u)\le e^{\frac{u^2}{2}}$ for all $u\in \R$, the Lipschitz assumption implies that
\[\cosh(\frac{\beta}{2}(f(x)-f(y)))\le e^{\frac{\beta^2L^2}{8}d(x,y)^2}.\]
As a result, it is enough to have $\beta^2L^2\le \alpha\beta$,  i.e., $\beta\le \alpha/L^2$, as required.

\subsection{Proof of Theorem \ref{thm:ewb}}

We start with three preliminary results. The first lemma below follows the lines devised in \citet{GyoOtt07} and is reported for completeness. 
\begin{lem}
\label{lem:ewb1}
Provided $0<\beta_{t+1}\le\beta_t$, for all $t\ge 1$, we have for all $n\ge 1$,
\begin{equation}
\label{thm:ewb:e1}
    -\sum_{t=1}^n\frac{1}{\beta_t}\ln\left(\int_{M}e^{-\beta_t\ell_t}\,\mathrm{d}m_t\right)\le -\frac{1}{\beta_{n+1}}\ln\left(\int_{M}e^{-\beta_{n+1}L_n}\,\mathrm{d}m\right),
\end{equation}
where $L_n:=\sum_{t=1}^n\ell_t$.
\end{lem}
\begin{proof}[Proof of Lemma \ref{lem:ewb1}]
For all $t\ge 1$, denote
\[W_t:=\int_{M}e^{-\beta_t L_{t-1}}\,\mathrm{d}m,\]
with the convention that $L_0(x)=0$, for all $x\in M$. Since $W_1=1$ we get
\begin{align}
\frac{\ln W_{n+1}}{\beta_{n+1}}
 &=\sum_{t=1}^n\left(\frac{\ln W_{t+1}}{\beta_{t+1}}-\frac{\ln W_{t}}{\beta_{t}}\right)
 \nonumber\\
   &=\sum_{t=1}^n\left(\frac{\ln W_{t+1}}{\beta_{t+1}}-\frac{\ln W^*_{t+1}}{\beta_{t}}\right)+\sum_{t=1}^n \frac{1}{\beta_t}\ln\left(\frac{W^*_{t+1}}{W_t}\right),
   \label{thm:ewb:e2}
   \end{align}
   where 
\[W^*_{t+1}:=\int_{M}e^{-\beta_{t} L_{t}}\,\mathrm{d}m.\]
Since $0<\beta_{t+1}\le \beta_t$, Jensen's inequality implies that 
\[W_{t+1}=\int_{M}e^{-\beta_{t+1} L_{t}}\,\mathrm{d}m\le\left( \int_{M}e^{-\beta_{t} L_{t}}\,\mathrm{d}m\right)^\frac{\beta_{t+1}}{\beta_t}=(W^*_{t+1})^\frac{\beta_{t+1}}{\beta_t},\]
which shows that the first sum in expression \eqref{thm:ewb:e2} is non-positive. Hence, we get
\begin{align}
\frac{\ln W_{n+1}}{\beta_{n+1}}&\le \sum_{t=1}^n \frac{1}{\beta_t}\ln\left(\frac{W^*_{t+1}}{W_t}\right)
\nonumber\\
&=\sum_{t=1}^n\frac{1}{\beta_t}\ln\int_{M}e^{-\beta_t\ell_t}\,\mathrm{d}m_t,
   \nonumber
   \end{align}
which concludes the proof.
\end{proof}
The next Lemma is also known in the Euclidean context but holds in our setting. Below, we denote $\mathcal E(\mu|m)$ the relative entropy of $\mu$ with respect to $m$, i.e., 
   \[\mathcal{E}(\mu|m):=\int_M \ln\left(\frac{{\rm d}\mu}{{\rm d}m}\right){\rm d}\mu,\]
if $\mu\ll m$, and $\mathcal{E}(\mu|m):=+\infty$ otherwise.

\begin{lem}
\label{lem:ewb2}
Assume that \ref{A1} and \ref{A2} hold. Suppose that there exists $\beta>0$ such that every $\ell\in\mathcal L$ is geodesically $\beta$-expconcave. Then the cumulative loss of the \textsc{ewb} forecaster, with $\beta_t=\beta$ for all $t\ge 1$, satisfies
\[\sum_{t=1}^n\ell_t(x_t)\le \inf_{\mu}\left\{\int_M L_n\,\mathrm{d}\mu+\frac{\mathcal E(\mu|m)}{\beta}\right\},\]
where $L_n:=\sum_{t=1}^n\ell_t$ and where the inf runs over all probability measures $\mu$ on $M$. 
\end{lem}

\begin{proof}[Proof of Lemma \ref{lem:ewb2}]
By concavity of $\exp(-\beta\ell_t)$, Assumption \ref{A2} implies that
   \[\ell_t(x_t)\le-\frac{1}{\beta}\ln \int_{M} e^{-\beta \ell_t}\mathrm{d}m_t.\] 
   Summing over $1\le t\le n$ and using Lemma \ref{lem:ewb1}, it follows that 
   \begin{equation}
    \nonumber
   \sum_{t=1}^n\ell_t(x_t)\le -\frac{1}{\beta}\ln\left(\int_{M}e^{-\beta L_n}\,\mathrm{d}m\right).
   \end{equation}
   It remains to observe that, according to the Gibbs variational principle, we have
   \[-\frac{1}{\beta}\ln\left(\int_{M}e^{-\beta L_n}\,\mathrm{d}m\right)=\inf_{\mu}\left\{\int_M L_n\,\mathrm{d}\mu+\frac{\mathrm{Ent}(\mu|m)}{\beta}\right\}.\]
\end{proof}

\begin{lem}
\label{lem:ewb3}
\begin{enumerate}[label=\rm{(\arabic*)}]
    \item For all $0\le r\le \pi$ and all $\e\in(0,1)$,
\[\frac{\sin(\e r)}{\sin(r)}\ge \e,\]
with the convention $0/0=1$.
\item For all $r\ge 0$ and all $\e\in(0,1/2]$,
\[\frac{\sinh(\e r)}{\sinh(r)}\ge \e \psi(r) ,\]
with $\psi(r):=(r \coth r)\exp(-r \coth r)$ and the convention $0/0=1$.
\end{enumerate}
\end{lem}
\begin{proof}[Proof of Lemma \ref{lem:ewb3}]
$(1)$ The statement follows from the concavity of the sine function on $[0,\pi]$. $(2)$ A direct computation shows that, for any $r>0$, the map $g:(0,1]\to\R$ defined by \[g(\e):=\log \sinh(\e r),\]
is concave. Hence, for all $\e\in(0,1]$, we have \[g(\e)-g(1)\ge-(1-\e)g'(\e)=-\e(1-\e)f(r)\ge -\e f(r),\]
where we denote $f(r):=r \coth r$. Therefore, since $1+x\le e^x$ for $x\ge 0$, it follows that
\begin{align*}
    \frac{\sinh(\e r)}{\sinh(r)}&=\exp(g(\e)-g(1))\\
    &\ge\exp(-\e f(r))\\
    &= \exp((1-\e) f(r))\exp(- f(r))\\
    &\ge (1+(1-\e)f(r))\exp(- f(r))\\
    &\ge \e f(r)\exp(- f(r)),
\end{align*}
where the last inequality uses the fact that $\e\le 1/2$.
\end{proof}

\begin{proof}[Proof of Theorem \ref{thm:ewb}]
Fix $n\ge2$ and   
   \[x^*_n\in\underset{x\in M}{\arg\min} \,L_n(x).\]
   Let $h:M\times[0,1]\to M$ be a geodesic homothety with center $x^*_n$ for which the measure contraction property \ref{A3} holds. For $\e\in(0,1)$, denote
   \[M_{\e}:=M^{\e}_{x^*_n}=\{h(y,\e):y\in M\},\]
   and introduce the probability measure
   \[m_{\e}(.):=\frac{m(.\cap M_{\e})}{m(M_{\e})}.\]
   According to Lemma \ref{lem:ewb2} we deduce that 
   \begin{align}
       \sum_{t=1}^n \ell_t(x_t)&\le \int_{M} L_n\mathrm{d}m_{\e}+\frac{\mathcal E(m_{\e}|m)}{\beta}
       \nonumber\\
       &=\frac{1}{m(M_{\e})}\int_{M_{\e}} L_n\mathrm{d}m+\frac{1}{\beta}\ln \left(\frac{1}{m(M_{\e})}\right).
       \label{thm:ewb:e3}
   \end{align}
Then, since $h(y,.):[0,1]\to M$ is a geodesic connecting $x^*_n$ to $y$, the $\beta$-expconcavity of the losses implies that  
   \[e^{-\beta \ell_t(h(y,\e))}\ge (1-\e)e^{-\beta \ell_t(x^*_n)}+\e e^{-\beta \ell_t(y)}\ge (1-\e)e^{-\beta \ell_t(x^*_n)},\]
   and therefore
   \[e^{-\beta L_n(h(y,\e))}\ge (1-\e)^ne^{-\beta L_n(x^*_n)},\]
   which implies
   \[L_n(h(y,\e))\le \frac{n}{\beta}\ln\left(\frac{1}{1-\e}\right)+L_n(x^*_n).\]
Combining this inequality with \eqref{thm:ewb:e3}, we obtain 
\begin{equation}
\nonumber
R_n\le\frac{n}{\beta}\ln\left(\frac{1}{1-\e}\right)+\frac{1}{\beta}\ln \left(\frac{1}{m(M_{\e})}\right).
\end{equation}
It remains to bound $m(M_{\e})$ from below. According to \ref{A3} we have
\[m(M_{\e})\ge\int_{M}\e\left(\frac{s_{\kappa}(\e d(x^{*}_n,y)/\sqrt{p-1})}{s_{\kappa}(d(x^*_n,y)/\sqrt{p-1})}\right)^{p-1}\,m(\mathrm{d}y).\]
To make this lower bound more explicit, we consider three cases.\\

\noindent$\bullet$ Suppose $\kappa>0$. Then, according to the Bonnet-Myers Theorem, valid under the \textsc{mcp} property~\citep[Theorem 4.3]{Oht07}, the diameter of $M$ is at most $\pi\sqrt{(p-1)/\kappa}$. Hence, it follows from  Lemma \ref{lem:ewb3} point $(1)$, and the definition of $s_{\kappa}$ in display \eqref{eq:skappa}, that for all $\e\in(0,1)$, $m(M_{\e})\ge \e^p$ which implies that 
\[R_n\le\frac{n}{\beta}\ln\left(\frac{1}{1-\e}\right)+\frac{p}{\beta}\ln\frac{1}{\e}.\]
Taking $\e=1/n$, and using the standard inequality $\ln(1+x)\le x$, we obtain in particular 
\[R_n\le \frac{n}{\beta}\ln\left(\frac{n}{n-1}\right)+\frac{p\ln n}{\beta}\le \frac{n}{\beta(n-1)}+\frac{p\ln n}{\beta},
\]
as desired.\\

\noindent$\bullet$ Suppose $\kappa=0$. Then since $s_{\kappa}(r)=r$, the \textsc{mcp} property reads in this case $m(M_{\e})\ge \e^p$ and we conclude as above.\\ 

\noindent$\bullet$ Suppose $\kappa<0$. Then Lemma \ref{lem:ewb2} point $(2)$ implies that, for all $\e\in(0,1/2]$
\begin{align*}
m(M_{\e})&\ge \e^p \int_M \psi\left(d(x^*_n,y)\sqrt{\frac{-\kappa}{p-1}}\right)\,m(\mathrm{d}y)\\
&=c_{\kappa,p}(x^*_{n})\e^p.
\end{align*}
Hence, for all $\e\in(0,1/2]$,
\[R_n\le \frac{n}{\beta}\ln\left(\frac{1}{1-\e}\right)+\frac{1}{\beta}\ln\frac{1}{c_{\kappa,p}(x^*_n)}+\frac{p}{\beta}\ln\frac{1}{\e}.\]
Hence taking $\e=1/n$ allows to conclude as in the previous cases.
\end{proof}

\subsection{Proof of Proposition \ref{pro:ewb2}}

Consider
   \[x^*_n\in\underset{x\in M}{\arg\min} \,L_n(x). \]
   For $r\in(0,r_0]$, define the probability measure 
   \[\mu_{r}(.):=\frac{m(.\cap B(x^*_n,r))}{m(B(x^*_n,r))}.\]
   According to Lemma \ref{lem:ewb2} we deduce that 
   \begin{align}
       \sum_{t=1}^n \ell_t(x_t)&\le \int_{M} L_n\mathrm{d}\mu_{r}+\frac{\mathcal E(\mu_r|m)}{\beta}
       \nonumber\\
       &=\frac{1}{m(B(x^*_n,r))}\int_{B(x^*_n,r)} L_n\mathrm{d}m+\frac{1}{\beta}\ln \left(\frac{1}{m(B(x^*_n,r))}\right)
       \nonumber\\
       &\le\frac{1}{m(B(x^*_n,r))}\int_{B(x^*_n,r)} L_n\mathrm{d}m+\frac{1}{\beta}\ln\frac{1}{c(x^*_n)}+\frac{p}{\beta}\ln\frac{1}{r},
       \nonumber
   \end{align}   
   by Assumption \eqref{eq:mesballs}. Substracting $L_n(x^*_n)$ on both sides and using the Lipschitz property of the losses, we obtain, for all $r\in(0,r_0]$,
   \begin{align*}
       R_n&\le nLr+\frac{1}{\beta}\ln\frac{1}{c(x^*_n)}+\frac{p}{\beta}\ln\frac{1}{r}.
   \end{align*}
   Taking finally $r=1/n$ implies the desired result.
   
\subsection{Proof of Proposition \ref{pro:mcplessgen}}
According to~\citet[Lemma 2.5]{Oht07}, for all $0<r\le r_0$ (and $r_0\le \pi\sqrt{(p-1)/\kappa}$ if $\kappa>0$), we have, for all $x\in M$,
\[\frac{m(B(x,r))}{m(B(x,r_0))}\ge \frac{r}{r_0} \inf_{\e\in[0,1]}\left(\frac{s_{\kappa}(\e r/\sqrt{p-1})}{s_{\kappa}(\e r_0/\sqrt{p-1})}\right)^{p-1}.\]
In particular, for all $0<r\le r_0$ (and $2r_0\le \pi\sqrt{(p-1)/\kappa}$ if $\kappa>0$), Lemma \ref{lem:ewb3} and the fact that the map $\psi(u)=(u \coth u)\exp(-u\coth u)$ is decreasing on $\R_+$ (see Figure \ref{fig:psi}), we obtain
\[\inf_{\e\in[0,1]}\left(\frac{s_{\kappa}(\e r/\sqrt{p-1})}{s_{\kappa}(\e 2r_0/\sqrt{p-1})}\right)^{p-1}\ge c(\kappa,p,r_0) \left(\frac{r}{2r_0}\right)^{p-1},\]
where constant $c(\kappa,p,r_0)=1$ if $\kappa\ge 0$ and 
\[c(\kappa,p,r_0):=\psi\left(2r_0\sqrt{\frac{-\kappa}{p-1}}\right)^{p-1},\]
if $\kappa<0$. In particular, for all $0<r\le r_0$ (and $2r_0\le \pi\sqrt{(p-1)/\kappa}$ if $\kappa>0$), we obtain
\[m(B(x,r))\ge c(x)r^p\quad\mbox{with}\quad c(x):= c(\kappa,p,r_0)\frac{m(B(x,2r_0))}{2^pr_0^p},\]
which proves the claim.

\subsection{Proof of Theorem \ref{thm:otbc}}
First, since inequality \eqref{eq:boundforotbc} holds for any outcome sequence, we have (almost surely) 
\[\frac{1}{n+1}\sum_{t=1}^{n+1}\ell(\theta_t,Z_t)-\inf_{\theta\in M}\frac{1}{n+1}\sum_{t=1}^{n+1}\ell(\theta,Z_t)\le \frac{B_{n+1}}{n+1}.\]
Taking the expectation on both sides we deduce that
\begin{align}
    \frac{1}{n+1}\sum_{t=1}^{n+1}\esp[\ell(\theta_t,Z_t)]&\le\esp[\inf_{\theta\in M}\frac{1}{n+1}\sum_{t=1}^{n+1}\ell(\theta,Z_t)]+\frac{B_{n+1}}{n+1}
    \nonumber\\
    &\le\inf_{\theta\in M}\frac{1}{n+1}\sum_{t=1}^{n+1}\esp[\ell(\theta,Z_t)]+\frac{B_{n+1}}{n+1}
    \nonumber\\
    & =\inf_{\theta\in M}\esp[\ell(\theta,Z)]+\frac{B_{n+1}}{n+1}.
    \label{eq:otbc:1}
\end{align}
Assumption \ref{A2} and the definition of $\hat\theta_n$ imply that, for all $z\in\mathcal Z$, 
\begin{equation}
\nonumber
\ell(\hat\theta_n,z)\le \frac{1}{n+1}\sum_{t=1}^{n+1}\ell(\theta_t,z).
\end{equation}
In particular,
\begin{align}
    \esp[\ell(\hat\theta_n,Z)]&\le \frac{1}{n+1}\sum_{t=1}^{n+1}\esp[\ell(\theta_t,Z)]
    \nonumber\\
   & = \frac{1}{n+1}\sum_{t=1}^{n+1}\esp[\ell(\theta_t,Z_t)],
   \label{eq:otbc:2}
\end{align}
where the last identity holds since $\theta_t=\vartheta_t(Z_1,\dots,Z_{t-1})$ and $Z_t$ are independent for all $t\ge 1$. The proof then follows by combining \eqref{eq:otbc:1} and \eqref{eq:otbc:2}.

\subsection{Proof of Lemma \ref{lem:w2mixable}}
We use the following technical Lemma due to \citet{Che20}, refining previous results from \citet{AhiGouPar19} and \citet{GouParRigStr19}. Recall that the variance of $P\in\mathcal P_2(\mathcal P_2(\R^d))$ is defined by
\[\mathcal V^*_P:=\inf_{\mu\in\mathcal P_2(\R^d)}\int W^2_2(\mu,.)\,\mathrm{d}P.\]
\begin{lem}[\citealp{Che20}, Theorem 6]
\label{lem:vi}
Let $P\in\mathcal P_2(\mathcal P_2(\R^d))$ and $\mu^*$ be a barycenter of $P$. Suppose that, for $P$-almost all $\nu$, $\varphi_{\mu^*\to\nu}$ is $\alpha(\nu)$-strongly convex for a measurable $\alpha:\mathcal P_2(\R^d)\to \R_+$. Then $\mu^*$ is unique and, for all $\mu\in\mathcal P_2(\R^d)$,
\[C_{\rm var}W^2_2(\mu^*,\mu)\le\int W^2_2(\mu,.)\,\mathrm{d}P-\mathcal V^*_P,\]
for $C_{\rm var}:=\int \alpha(\nu)\,\mathrm{d}P(\nu)$.
\end{lem}
Lemma \ref{lem:vi} is remarkable since it provides, for all $\mu$, a distribution dependent analog of Jensen's inequality for $W^2_2(\mu,.)$ (see Lemma \ref{lem:jensen2}) while this function isn't geodesically convex. Next, we show how this result implies directly Lemma \ref{lem:w2mixable} under slightly more general conditions. For readability, we use notation 
\[\mathcal E_{\mu}(\nu):=W^2_2(\mu,\nu)-\int W^2_2(\mu,.)\mathrm{d}P.\]
\begin{lem}
\label{lem:w2mixable:gen}
Under the assumptions of Lemma \ref{lem:vi}, for all $\mu\in\mathcal P_2(\R^D)$, and all $\beta>0$ satisfying
\begin{equation}
\label{eq:concentrationp}
\frac{1}{\beta}\ln\int\exp(-\beta\mathcal E_{\mu})\,\mathrm{d}P\le \mathcal V^*_P,
\end{equation}
we have
\begin{equation}
\label{eq:vi}
W^2_2(\mu^*,\mu)\le -\frac{C^{-1}_{\rm var}}{\beta}\ln\int \exp(-\beta W^2_2(\mu,.))\,\mathrm{d}P.
\end{equation}
\end{lem}

\begin{proof}[Proof of Lemma \ref{lem:w2mixable:gen}] It suffices to notice that, according to Lemma \ref{lem:vi}, we obtain 
\begin{align*}
    \ln\int \exp(-\beta W^2_2(\mu,.))\,\mathrm{d}P&=-\beta\int W^2_2(\mu,.)\mathrm{d}P+\ln\int \exp(-\beta \mathcal E_{\mu})\mathrm{d}P\\
    &\le -\beta C_{\rm var}W^2_2(\mu^*,\mu) -\beta\mathcal V^*_P +\ln\int \exp(-\beta \mathcal E_{\mu})\mathrm{d}P,
\end{align*}
which implies the claim.
\end{proof}

To prove Lemma \ref{lem:w2mixable}, it remains to observe that, if $P$ has bounded support, then by Hoeffding's Lemma we have, for any $\beta>0$, \[\frac{1}{\beta}\ln\int\exp(-\beta\mathcal E_{\mu})\,\mathrm{d}P\le \frac{\beta\Delta^4}{8}.\]

\appendix
\section{Omitted proofs}
\subsection{Proof of Theorem \ref{thm:ewb3}}
The proof follows the same strategy as that of Theorem \ref{thm:ewb}. We only need to slightly adapt Lemma \ref{lem:ewb2}.

\begin{lem}
\label{lem:ewb2bis}
Assume that \ref{A1} and \ref{A2} hold. Suppose that every $\ell\in\mathcal L$ is geodesically convex. Then the cumulative loss of the \textsc{ewb} forecaster, with $0<\beta_{t+1}<\beta_t$ for all $t\ge 1$, satisfies
\[\sum_{t=1}^n\ell_t(x_t)\le \inf_{\mu}\left\{\int_M L_n\,\mathrm{d}\mu+\frac{\mathcal{E}(\mu|m)}{\beta_{n+1}}\right\}+\sum_{t=1}^n\frac{1}{\beta_t}\ln\left(\int_{M}e^{\beta_t(\bar\ell_t-\ell_t)}\mathrm{d}m_t\right),\]
where $\bar\ell_t:=\int_M\ell_t\mathrm{d}m_t$, where $L_n=\sum_{t=1}^n\ell_t$ and where the inf runs over all probability measures $\mu$ on $M$.
\end{lem}

\begin{proof}[Proof of Lemma \ref{lem:ewb2}]
By geodesic convexity of $\ell_t$ and Assumption \ref{A2}, we obtain
 \begin{align*}
     \ell_t(x_t)-\frac{1}{\beta_t}\ln \int e^{\beta_t(\bar\ell_t-\ell_t)}\mathrm{d}m_t& \le \bar\ell_t-\frac{1}{\beta_t}\ln \int e^{\beta_t(\bar\ell_t-\ell_t)}\mathrm{d}m_t\\
     &= -\frac{1}{\beta_t}\ln \int e^{-\beta_t\ell_t}\mathrm{d}m_t.
 \end{align*}
Summing over $1\le t\le n$ and using Lemma \ref{lem:ewb1}, we obtain
\[\sum_{t=1}^n\ell_t(x_t)\le -\frac{1}{\beta_{n+1}}\ln\left(\int_{M}e^{-\beta_{n+1} L_n}\,\mathrm{d}m\right)+\sum_{t=1}^n\frac{1}{\beta_t}\ln\left(\int_{M}e^{\beta_t(\bar\ell_t-\ell_t)}\mathrm{d}m_t\right).\]
We conclude, as in Lemma \ref{lem:ewb2}, by using the Gibbs variational principle.
\end{proof}

\begin{proof}[Proof of Theorem \ref{thm:ewb3}]
Using Lemma \ref{lem:ewb2bis}, and letting $m_{\e}$ be as in the proof of Theorem \ref{thm:ewb}, we obtain for all $\e\in(0,1)$,
   \begin{align}
       &\sum_{t=1}^n \ell_t(x_t)
       \nonumber\\
       &\le \int_{M} L_n\mathrm{d}m_{\e}+\frac{\mathcal E(m_{\e}|m)}{\beta_{n+1}}+\sum_{t=1}^n\frac{1}{\beta_t}\ln\left(\int_{M}e^{\beta_t(\bar\ell_t-\ell_t)}\mathrm{d}m_t\right)
       \nonumber\\
       &=\frac{1}{m(M_{\e})}\int_{M_{\e}} L_n\mathrm{d}m+\frac{1}{\beta_{n+1}}\ln \left(\frac{1}{m(M_{\e})}\right)+\sum_{t=1}^n\frac{1}{\beta_t}\ln\left(\int_{M}e^{\beta_t(\bar\ell_t-\ell_t)}\mathrm{d}m_t\right).
       \nonumber
   \end{align}
   By convexity and boundedness of the losses, we deduce that, for all $y\in M$,
   \begin{align*}
   L_n(h(y,\e))&\le (1-\e)L_n(x^*_n)+\e L_n(y)\\
   &\le L_n(x^*_n)+ \e n(b-a).
   \end{align*}
Combining this inequality with the above, we obtain 
\begin{align*}
R_n&\le \e n(b-a)+\frac{1}{\beta_{n+1}}\ln \left(\frac{1}{m(M_{\e})}\right)+\sum_{t=1}^n\frac{1}{\beta_t}\ln\left(\int_{M}e^{\beta_t(\bar\ell_t-\ell_t)}\mathrm{d}m_t\right)\\
&\le \e n(b-a)+\frac{1}{\beta_{n+1}}\ln \left(\frac{1}{m(M_{\e})}\right)+\frac{(b-a)^2}{8}\sum_{t=1}^n\beta_t,
\end{align*}
where the last inequality follows from Hoeffding's Lemma. Now taking $\e=1/n$, we deduce as in the proof of Theorem \ref{thm:ewb} that 
\[R_n\le (b-a)+\frac{1}{\beta_{n+1}}\ln\frac{1}{c_{\kappa,p}(x^*_n)}+\frac{p\ln n}{\beta_{n+1}}+\frac{(b-a)^2}{8}\sum_{t=1}^n\beta_t.\]
Hence, for $c>0$, taking 
\[\beta_t=\frac{c}{(b-a)}\sqrt{\frac{p\ln(t\lor 2)}{t}},\]
we obtain, for all $n\ge 2$, 
\[
    \sum_{t=1}^n\beta_t\le \frac{c}{(b-a)}\sqrt{p\ln n}\sum_{t=1}^n\frac{1}{\sqrt{t}}\le\frac{2c\sqrt{pn\ln n}}{(b-a)}.\]
Using the fact that $(n+1)/\ln(n+1)\le (3/2)(n/\ln n)$ for all $n\ge 2$, and taking $c=2(3/2)^{1/4}$, so that the two leading terms coincide, we obtain
\[R_n\le (b-a)\left[1+\left(\frac{3}{2}\right)^{\frac{1}{4}}\left(\ln\frac{1}{c_{\kappa,p}(x^*_n)}\sqrt{\frac{n}{p\ln n}}+\sqrt{pn\ln n}\right)\right],\]
which completes the proof.
\end{proof}

\bibliography{ewb}
\end{document}